\documentclass[letterpaper, 10 pt, conference]{ieeeconf}
\IEEEoverridecommandlockouts

\usepackage[colorlinks]{hyperref}
\hypersetup{
    colorlinks,
    linkcolor=black,
    citecolor=black,
    filecolor=magenta,
    urlcolor=cyan,
}

\usepackage{resizegather}

\usepackage{cite}
\usepackage{amsmath,amssymb,amsfonts,mathrsfs}
\usepackage{mdframed}
\usepackage{graphicx}
\usepackage{textcomp}
\usepackage{xcolor}
\usepackage{tcolorbox}
\usepackage{textcomp}
\usepackage{tablefootnote}
\usepackage{threeparttable}
\usepackage{caption, subcaption}
\usepackage{bbm}
\usepackage{dsfont}
\usepackage{tikz}
\usepackage{graphicx}
\usepackage{tkz-euclide}
\usepackage{algorithm}
\usepackage{algpseudocode}
\usepackage{enumerate}
\usepackage{mathtools}
\usepackage{balance}
\usetikzlibrary{positioning}
\usetikzlibrary{shapes}
\usetikzlibrary{shapes.misc}
\usetikzlibrary{shapes.geometric}
\usetikzlibrary{plotmarks}
\usetikzlibrary{intersections}
\usetikzlibrary{calc}
\usetikzlibrary{fit}
\usetikzlibrary{patterns,tikzmark}
\usetikzlibrary{matrix,decorations.pathreplacing,calc}

\tikzset{cross/.style={cross out, draw, 
         minimum size=2*(#1-\pgflinewidth), 
         inner sep=0pt, outer sep=0pt}}

\newcommand{\vertiii}[1]{{\left\vert\kern-0.25ex\left\vert\kern-0.25ex\left\vert #1 
    \right\vert\kern-0.25ex\right\vert\kern-0.25ex\right\vert}}

\newmdtheoremenv[linecolor=white, backgroundcolor=lightgray!15, innertopmargin=5pt, innerbottommargin=5pt, skipabove=10pt, skipbelow=10pt]{theorem}{\textbf{Theorem}}
\newmdtheoremenv[linecolor=white, backgroundcolor=lightgray!15, innertopmargin=5pt, innerbottommargin=5pt, skipabove=10pt, skipbelow=10pt]{corollary}{\textbf{Corollary}}[theorem]
\newmdtheoremenv[linecolor=white, backgroundcolor=lightgray!15, innertopmargin=5pt, innerbottommargin=5pt, skipabove=10pt, skipbelow=10pt]{lemma}{\textbf{Lemma}}

\newmdtheoremenv[linecolor=white, backgroundcolor=lightgray!15, innertopmargin=5pt, innerbottommargin=5pt, skipabove=10pt, skipbelow=10pt]{problem}{\textbf{Problem}}

\newmdtheoremenv[linecolor=white, backgroundcolor=lightgray!15, innertopmargin=5pt, innerbottommargin=5pt, skipabove=10pt, skipbelow=10pt]{definition}{\textbf{Definition}}

\newmdtheoremenv[linecolor=white, backgroundcolor=lightgray!15, innertopmargin=5pt, innerbottommargin=5pt, skipabove=10pt, skipbelow=10pt]{objective}{\textbf{Objective}}

\newmdenv[
    linecolor=white, backgroundcolor=lightgray!15, innertopmargin=5pt, innerbottommargin=5pt, skipabove=10pt, skipbelow=10pt
]{graybox}

\makeatletter
\renewcommand{\fps@figure}{htp}
\renewcommand{\fps@table}{htp}
\makeatother

\def\BibTeX{{\rm B\kern-.05em{\sc i\kern-.025em b}\kern-.08em
    T\kern-.1667em\lower.7ex\hbox{E}\kern-.125emX}}

\title{HJ-SafeDMP: Hamilton-Jacobi Reachability-Guided Dynamic Movement Primitives for Provably Safe Robot Motion}
\author{Siddhanth Ramesh and Ravi Prakash
\thanks{All the authors belong to Cyber-Physical Systems, Indian Institute of Science (IISc), Bengaluru.
{\tt siddhanthramesh587@gmail.com, ravipr@iisc.ac.in}.}
}

\begin{document}

\maketitle

\begin{abstract}
Robots deployed in safety-critical environments must execute motions that are simultaneously robust to disturbances and provably safe from collisions. Dynamic Movement Primitives (DMPs) offer inherent stability, temporal flexibility, and efficient trajectory generalization from single demonstrations, but they lack formal safety certificates. Conversely, Hamilton-Jacobi (HJ) Reachability analysis provides a principled framework for computing worst-case safety margins and forward-invariant safe sets, but classical grid-based methods suffer from the curse of dimensionality and are impractical for real-time control. This paper introduces \textbf{HJ-SafeDMP}, a framework that integrates DMPs with learned HJ Reachability-based safety value functions to achieve provably safe, robust, and computationally efficient robot motion. We learn a Control Barrier Value Function (CBVF) from offline demonstration data using a model-free, finite-difference HJ recursion and deploy it as a real-time safety filter via a closed-form control law that modulates the DMP output. Unlike optimization-based CBF-QP approaches, our method achieves safety filtering without online quadratic program solves, preserving the computational efficiency of DMPs. We further incorporate an expectile-based offline learning objective that avoids querying out-of-distribution actions, and a conformal prediction calibration step that provides finite-sample probabilistic safety coverage. Experimental evaluation on a 7-DOF robot manipulator demonstrates that HJ-SafeDMP achieves formal safety guarantees with orders-of-magnitude faster execution than optimization-based baselines, while maintaining the robustness and adaptability of DMPs for human-robot interaction.
\end{abstract}


\section{Introduction}
\label{section: introduction}

For robots to operate reliably in human-centric environments, they must satisfy two fundamental properties: \emph{robustness}---the capacity to maintain stable, goal-directed behavior under external disturbances---and \emph{safety}---the formal guarantee of avoiding collisions with people, objects, and environmental boundaries~\cite{brunke2022safe}. While both properties are essential for real-world deployment, achieving them simultaneously in a computationally efficient manner remains a central challenge in robotic control~\cite{siciliano2009robotics}.

Dynamic Movement Primitives (DMPs)~\cite{ijspeert2013dynamical, schaal2006dynamic} have emerged as a powerful framework for learning and generalizing robot motions from demonstrations~\cite{ravichandar2020recent}. By encoding trajectories as stable, attractor-based dynamical systems, DMPs provide inherent robustness to perturbations, temporal flexibility, and spatial generalization from a single demonstration~\cite{pastor2009learning, saveriano2023dynamic}. These properties make DMPs particularly attractive for human-robot interaction (HRI), where motions must adapt in real time to changing conditions. However, DMPs fundamentally lack a principled mechanism for enforcing safety. They are typically paired with reactive, heuristic methods such as Artificial Potential Fields (APFs)~\cite{hoffmann2009biologically, khatib1986real}, which are susceptible to local minima and oscillations, providing no formal safety guarantees.

On the other hand, control-theoretic methods provide rigorous foundations for safety. Control Barrier Functions (CBFs)~\cite{ames2017CBF, Ames2019CBF, wieland2007constructive} guarantee forward invariance of a safe set through constrained optimization, but require solving a Quadratic Program (QP) at every control step---imposing significant computational overhead that hinders high-frequency control~\cite{taylor2020learning}. More fundamentally, constructing valid CBFs for complex systems remains a challenging open problem~\cite{dawson2023safe}.

Hamilton-Jacobi (HJ) Reachability analysis~\cite{bansal2017hamilton, Fisac2019HJSafety, herbert2017fastrack} offers a principled framework for computing worst-case safety margins over future trajectories by solving a Hamilton-Jacobi-Bellman (HJB) partial differential equation. The resulting \emph{safety value function} naturally encodes forward-invariant safe sets and has been applied to safe planning in robotics~\cite{fisac2018general, chen2018decomposition}. When combined with CBFs, this yields the Control Barrier Value Function (CBVF)~\cite{choi2021robust}, which inherits formal safety guarantees while being amenable to learning-based approximation. Recent work has demonstrated that HJ-inspired value functions can be learned entirely from offline data using model-free Bellman recursions~\cite{tayal2026vocbf}, expectile regression~\cite{kostrikov2022offline}, and conformal prediction calibration~\cite{tayal2026safefql, lindemann2025formal}, bypassing the curse of dimensionality of grid-based methods~\cite{Mitchell2005ATO} while avoiding online interaction.

This paper introduces \textbf{HJ-SafeDMP}, a framework that resolves the trade-off between the computational efficiency of DMPs and the formal safety guarantees of HJ Reachability. Our key insight is to learn a CBVF from offline demonstration data via a finite-difference HJ recursion and deploy it as a \emph{closed-form safety modulation layer} that adjusts the DMP output in real time---without requiring online QP solves. The CBVF encodes worst-case safety margins that propagate information about future obstacle encounters backward in time, enabling proactive rather than reactive avoidance. This yields a system that is simultaneously robust (via DMPs), formally safe (via the learned CBVF), and computationally efficient (via closed-form control).

Our main contributions are:
\begin{enumerate}
    \item We propose \textbf{HJ-SafeDMP}, a novel framework that integrates Dynamic Movement Primitives with learned Hamilton-Jacobi Reachability-based safety value functions, achieving provably safe motion generation without online optimization.
    \item We derive a \textbf{model-free, finite-difference CBVF recursion} from HJ Reachability theory that enables learning safety certificates entirely from offline demonstration data, using expectile regression to handle out-of-distribution actions.
    \item We introduce a \textbf{closed-form safety modulation law} that filters DMP outputs using the learned CBVF gradient, replacing the computational overhead of CBF-QP with a single gradient evaluation per control step.
    \item We demonstrate through extensive experiments on a 7-DOF manipulator that HJ-SafeDMP achieves formal safety guarantees with orders-of-magnitude faster execution than optimization-based baselines, while maintaining the robustness and adaptability of DMPs.
\end{enumerate}

\section{Related Works}
\label{section: related_work}

\subsection{Motion Primitives for Robot Control}

Learning from Demonstration (LfD)~\cite{ravichandar2020recent} has produced a rich family of motion primitive frameworks for encoding and generalizing robot skills. Dynamic Movement Primitives (DMPs)~\cite{ijspeert2013dynamical, schaal2006dynamic} encode trajectories as attractor-based dynamical systems with inherent stability and generalization from single demonstrations, and have been widely adopted in manipulation~\cite{pastor2009learning} and locomotion tasks. A comprehensive survey of DMP variants is provided in~\cite{saveriano2023dynamic}. Alternative representations include Gaussian Mixture Models (GMMs)~\cite{jaquier2019learning}, Probabilistic Movement Primitives (ProMPs)~\cite{paraschos2013probabilistic}, Kernelized Movement Primitives (KMPs)~\cite{huang2017kernelized}, and stable dynamical systems learned via Lyapunov constraints~\cite{khansari2011learning}. More recent approaches employ neural ODEs for trajectory learning~\cite{nawaz2024learning} and diffusion-based generative models for policy representation~\cite{janner2022planning}. While these methods offer various trade-offs in expressiveness and data efficiency, DMPs remain uniquely attractive for their closed-form dynamics, temporal scaling, and robustness to perturbations. However, none of these frameworks provide formal safety guarantees; they are typically augmented with reactive avoidance strategies such as APFs~\cite{hoffmann2009biologically, khatib1986real}, which lack formal certificates and suffer from local minima.

\subsection{Safety in Robot Control}

Ensuring safety in robot control has been approached from multiple perspectives~\cite{brunke2022safe}. Formal safety methods based on Control Barrier Functions (CBFs)~\cite{ames2017CBF, Ames2019CBF, ames2014control, wieland2007constructive} guarantee forward invariance of safe sets through constrained optimization (CBF-QP). CBFs have been applied to adaptive cruise control~\cite{ames2014control}, bipedal walking~\cite{hsu2015control, nguyen2015safety}, aerial robotics~\cite{singletary2022onboard, tayal2024control}, and safe imitation learning~\cite{cosner2022end, ciftci2024safe}. Learning-based approaches for constructing CBFs include neural CBFs~\cite{dawson2022safe, dawson2023safe, Alexander2020NCBF}, physics-informed methods~\cite{tayal2025physics}, formally verified certificates~\cite{tayal2024learning, NEURIPS2023_120ed726, abate2021fossil}, and conservative offline approaches~\cite{tabbara2025CCBF, Fernando2023iDBF}. Frameworks combining dynamical systems with CBFs and CLFs~\cite{nawaz2024learning} provide provable guarantees but rely on computationally expensive online optimization, limiting their applicability to high-frequency control. Methods like ROAM~\cite{huber2023avoidance} and Spatio-Temporal Tubes~\cite{das2025spatiotemporal} offer closed-form alternatives but with less explicit safety margins or hand-designed boundaries. Model Predictive Control (MPC)~\cite{MAYNE2000789} provides another avenue for safe planning but incurs significant online computation.

\subsection{Hamilton-Jacobi Reachability for Safety}

HJ Reachability~\cite{bansal2017hamilton, Fisac2019HJSafety} provides the most principled framework for computing backward reachable sets and worst-case safety margins. It has been applied to safe motion planning via FaSTrack~\cite{herbert2017fastrack}, multi-agent systems~\cite{chen2018decomposition}, and general safety frameworks for learning-based control~\cite{fisac2018general}. The Control Barrier Value Function (CBVF)~\cite{choi2021robust} unifies HJ Reachability with CBFs, providing formal forward invariance guarantees. The state-constrained optimal control perspective has been explored via epigraph reformulations~\cite{Altarovici2013HJ, Hao2024csl}. Classical grid-based methods~\cite{Mitchell2005ATO} scale exponentially with dimension, motivating neural approximations. Recent work has demonstrated that HJ-inspired safety value functions can be learned from offline data using finite-difference recursions (V-OCBF~\cite{tayal2026vocbf}), reachability-inspired Bellman backups with flow policy distillation (SafeFQL~\cite{tayal2026safefql}), and epigraph reformulations with flow matching (EpiFlow~\cite{tayal2026epiflow}). These approaches use expectile regression~\cite{kostrikov2022offline} to avoid out-of-distribution action queries and achieve strong empirical safety. Conformal prediction~\cite{shafer2008tutorial, lindemann2025formal, tayal2025cp} has emerged as a tool for providing finite-sample probabilistic safety coverage. Offline RL methods~\cite{levine2020OffRL, kumar2020CQL} and safe offline RL benchmarks~\cite{liu2024dsrl, ji2023safety} provide the data and evaluation infrastructure for these approaches. Our work is the first to integrate such learned HJ Reachability-based safety certificates with Dynamic Movement Primitives for real-time safe robot motion.

\section{Background and Problem Setup}
\label{section: background}

\subsection{Dynamic Movement Primitives (DMPs)}
\label{sec:dmp}

DMPs~\cite{ijspeert2013dynamical, schaal2006dynamic} model robot trajectories as a second-order dynamical system with a learned nonlinear forcing term. They have been extensively used in manipulation~\cite{pastor2009learning}, locomotion, and human-robot interaction~\cite{saveriano2023dynamic}. The dynamics are governed by:
\begin{equation}
\tau^2 \ddot{x} = \alpha_z(\beta_z(g - x) - \tau \dot{x}) + f(z),
\label{eq:dmp}
\end{equation}
where $x(t) \in \mathbb{R}^d$ is the trajectory, $g \in \mathbb{R}^d$ is the goal, $\tau > 0$ is a temporal scaling factor, $\alpha_z, \beta_z > 0$ are critically-damped gains ($\beta_z = \alpha_z / 4$), and $f(z)$ is the forcing function parameterized by Gaussian basis functions, driven by a canonical phase variable:
\begin{equation}
\tau \dot{z} = -\alpha_z z, \quad z(0) = 1.
\label{eq:phase}
\end{equation}

DMPs provide: (i) \textbf{stable convergence} to the goal via attractor dynamics; (ii) \textbf{spatial scaling} by modifying $x_0$ and $g$; (iii) \textbf{temporal scaling} via $\tau$; and (iv) \textbf{robustness} to perturbations via the spring-damper structure. The forcing function $f(z)$ is learned from a single demonstration by inverting~\eqref{eq:dmp}.

\subsection{Control Barrier Functions (CBFs)}
\label{sec:cbf}

Given a continuously differentiable function $B: \mathcal{X} \to \mathbb{R}$, the safe set is $\mathcal{C} = \{x \in \mathcal{X} : B(x) \geq 0\}$.

\begin{definition}[Control Barrier Function~\cite{Ames_2017}]
For a control-affine system $\dot{x} = f(x) + g(x)u$, the function $B$ is a CBF on $\mathcal{X}$ if there exists an extended class-$\mathcal{K}$ function $\kappa$ such that:
\begin{equation}
\max_{u \in \mathbb{U}} \left[\mathcal{L}_f B(x) + \mathcal{L}_g B(x)\, u + \kappa(B(x))\right] \geq 0, \quad \forall\, x \in \mathcal{X},
\label{eq:cbf}
\end{equation}
where $\mathcal{L}_f B = \frac{\partial B}{\partial x} f(x)$ and $\mathcal{L}_g B = \frac{\partial B}{\partial x} g(x)$ are Lie derivatives.
\end{definition}

A standard CBF-QP safety filter minimally adjusts a reference controller:
\begin{equation}
\pi_{\mathrm{safe}}(x) = \min_{u \in \mathbb{U}} \|u - \pi_{\mathrm{ref}}(x)\|^2 \;\; \text{s.t.} \;\; \mathcal{L}_f B(x) + \mathcal{L}_g B(x)\, u + \kappa(B(x)) \geq 0.
\label{eq:cbf_qp}
\end{equation}

While this provides formal safety, solving~\eqref{eq:cbf_qp} at every control step imposes significant computational overhead~\cite{singletary2022onboard}, and constructing valid $B$ is non-trivial~\cite{dawson2023safe, taylor2020learning}.

\subsection{Hamilton-Jacobi Reachability and CBVFs}
\label{sec:hj}

HJ Reachability~\cite{bansal2017hamilton, Fisac2019HJSafety} computes the set of states from which safety can be guaranteed under worst-case conditions. It has been applied to safe motion planning~\cite{herbert2017fastrack}, multi-agent control~\cite{chen2018decomposition}, and learning-based safety~\cite{fisac2018general}. Given a safety specification $\ell: \mathbb{R}^n \to \mathbb{R}$ where $\mathcal{F} = \{x : \ell(x) < 0\}$ is the failure set, the \emph{Control Barrier Value Function (CBVF)}~\cite{choi2021robust} is defined as the viscosity solution of the HJB variational inequality:
\begin{equation}
\min\left\{\max_{u \in \mathbb{U}} \nabla B(x) \cdot (f(x) + g(x)\, u),\;\; \ell(x) - B(x)\right\} = 0.
\label{eq:hjb_vi}
\end{equation}

The CBVF $B(x)$ induces a forward-invariant safe set $\mathcal{C} = \{x : B(x) \geq 0\}$ and ensures that admissible controls satisfy:
\begin{equation}
\max_{u \in \mathbb{U}} \left[\mathcal{L}_f B(x) + \mathcal{L}_g B(x)\, u + \kappa(B(x))\right] \geq 0.
\label{eq:cbvf_condition}
\end{equation}

\textbf{Discrete-time recursion.} For learning from offline data, the CBVF admits a finite-difference approximation~\cite{tayal2026vocbf}:
\begin{equation}
B(x_t) = \min\left\{\ell(x_t),\; \max_{u_t} B(x_{t+1})\right\}, \quad \forall\, t \in \{0, 1, \ldots\}.
\label{eq:cbvf_recursion}
\end{equation}

This recursion encodes the principle that a state is safe if and only if its successor is safe or it lies outside the failure set.

\subsection{Problem Statement}

\begin{problem}[Safe DMP Execution]
Given a DMP learned from a single demonstration and an environment with static and dynamic obstacles, synthesize a control law $u(t)$ such that:
\begin{enumerate}
    \item The robot trajectory converges to the goal $g$ (robustness);
    \item The trajectory remains within the safe set $\mathcal{C} = \{x : B(x) \geq 0\}$ at all times (safety);
    \item The control law is computable in closed-form without online optimization (efficiency).
\end{enumerate}
\end{problem}

\section{Methodology}
\label{section: method}

We propose HJ-SafeDMP, a two-stage framework: (i) learning a safety value function from offline demonstrations using HJ Reachability-inspired recursions, and (ii) integrating this learned safety certificate with DMPs via a closed-form modulation law for real-time safe execution. Figure~\ref{fig:framework} illustrates the overall pipeline.

\begin{figure*}[t]
    \centering
    \includegraphics[width=\textwidth]{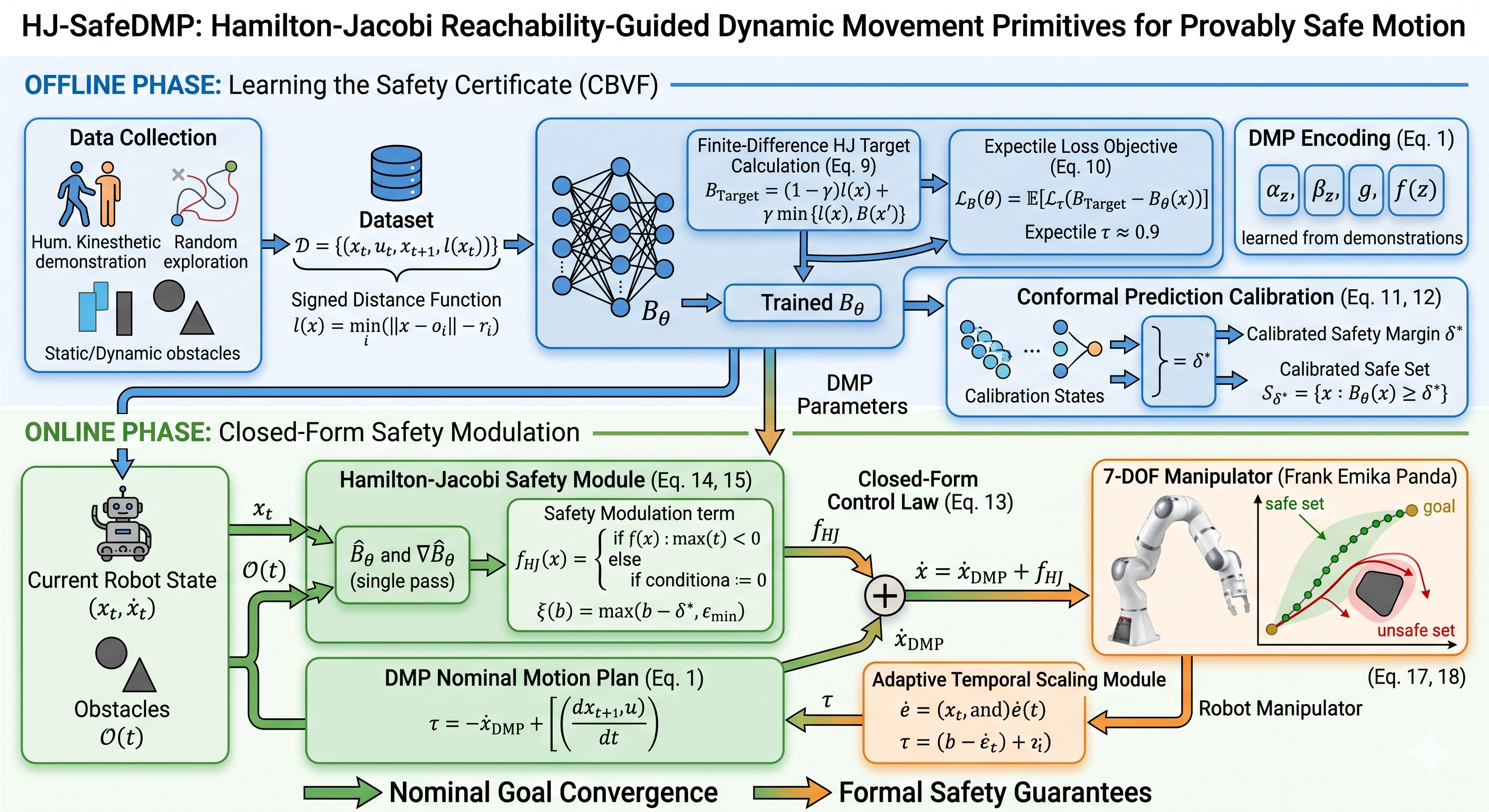}
    \caption{HJ-SafeDMP framework pipeline. (a)~Demonstrations are recorded and encoded as DMPs. (b)~A CBVF is learned offline from demonstration data via finite-difference HJ recursion with expectile regression. (c)~During execution, the learned CBVF modulates DMP outputs in closed form to enforce safety while preserving goal convergence.}
    \label{fig:framework}
\end{figure*}

\subsection{Stage 1: Learning the CBVF from Offline Data}
\label{sec:learning}

\subsubsection{Data Collection}
We collect an offline dataset $\mathcal{D} = \{(x_t, u_t, x_{t+1})\}$ from kinesthetic demonstrations and exploratory trajectories in the workspace. Each transition includes the end-effector state $x_t \in \mathbb{R}^d$, the applied control $u_t$, and the successor state $x_{t+1}$. We also compute the safety specification $\ell(x_t)$ for each state, encoding the signed distance to obstacles:
\begin{equation}
\ell(x) = \min_i \left(\|x - o_i\| - r_i\right),
\label{eq:safety_spec}
\end{equation}
where $o_i$ and $r_i$ are the center and radius of the $i$-th obstacle. States with $\ell(x) < 0$ are inside obstacles (unsafe).

\subsubsection{Finite-Difference CBVF Recursion}

Following the HJ Reachability framework (Sec.~\ref{sec:hj}), we approximate the continuous-time HJB-VI~\eqref{eq:hjb_vi} using the finite-difference barrier recursion from~\cite{tayal2026vocbf}. The discounted recursion target is:
\begin{equation}
B_{\mathrm{Target}}^{\mathcal{D}} = (1 - \gamma)\, \ell(x) + \gamma\, \min\left\{\ell(x),\; B_\psi(x')\right\},
\label{eq:target}
\end{equation}
where $x' = x_{t+1}$ is the observed successor, $\gamma \to 1$ is a discount factor ensuring contraction~\cite{Fisac2019HJSafety}, and $B_\psi$ is a neural network parameterizing the CBVF. This recursion propagates worst-case safety information backward through time: a state is marked unsafe if \emph{any} future state along the trajectory enters the failure set.

\subsubsection{Expectile-Based Learning Objective}

In the offline setting, maximizing over all possible actions in~\eqref{eq:cbvf_recursion} is infeasible because the dataset supports only a restricted action set. Following~\cite{tayal2026vocbf, kostrikov2022offline}, we employ \emph{expectile regression} to approximate the maximization over dataset-supported actions without querying out-of-distribution (OOD) inputs:
\begin{equation}
\mathcal{L}_B(\theta) = \mathbb{E}_{(x, u, x') \sim \mathcal{D}}\left[\mathcal{L}_\tau\left(B_{\mathrm{Target}}^{\mathcal{D}}(x) - B_\theta(x)\right)\right],
\label{eq:expectile_loss}
\end{equation}
where $\mathcal{L}_\tau(y) = |\tau - \mathds{1}(y < 0)|\, y^2$ is the $\tau$-expectile loss~\cite{ilya2022IQL}. For $\tau > 0.5$, this loss upweights positive residuals, causing $B_\theta$ to track the upper envelope of dataset-supported safety values. This yields a CBVF that reflects the \emph{best achievable safety} under data-supported actions, rather than the average behavior-policy safety.

\subsubsection{Conformal Prediction Calibration}

Due to finite-data approximation errors, the learned CBVF boundary $\{x : B_\theta(x) = 0\}$ may be miscalibrated~\cite{dawson2023safe}. Following~\cite{tayal2026safefql, shafer2008tutorial, lindemann2025formal}, we apply a conformal prediction calibration step~\cite{angelopoulos2022gentle, tayal2025cp}. We compute the rollout safety value $V_c^\pi(x_i)$ for a set of $N_s$ calibration states sampled from $\{x : B_\theta(x) \leq \delta\}$ and find the tightest correction margin:
\begin{equation}
\delta^* = \min_{\hat{x}_j} \left\{B_\theta(x_j) : V_c^\pi(x_j) \geq 0\right\},
\label{eq:conformal}
\end{equation}
such that, with probability $1 - \beta_s$:
\begin{equation}
\Pr_{x_i \in \mathcal{S}_{\delta^*}}\left(V_c^\pi(x_i) < 0\right) \geq 1 - \epsilon_s.
\end{equation}
The calibrated safe set $\mathcal{S}_{\delta^*} = \{x : B_\theta(x) \geq \delta^*\}$ provides finite-sample probabilistic safety coverage.

\subsection{Stage 2: HJ-SafeDMP --- Closed-Form Safety Modulation}
\label{sec:integration}

\subsubsection{DMP as Nominal Motion Plan}

The nominal motion plan is encoded as a DMP~\eqref{eq:dmp} learned from a single kinesthetic demonstration~\cite{ravichandar2020recent}. The DMP generates a reference trajectory $x_{\mathrm{DMP}}(t)$ and velocity $\dot{x}_{\mathrm{DMP}}(t)$ with inherent robustness to perturbations via the critically-damped attractor dynamics.

\subsubsection{CBVF-Based Safety Modulation}

Instead of solving a CBF-QP~\eqref{eq:cbf_qp} at each timestep, we use the gradient of the learned CBVF $B_\theta(x)$ to construct a \emph{closed-form safety modulation term} $f_{\mathrm{HJ}}$ that is additively integrated into the DMP dynamics:
\begin{equation}
\tau^2 \ddot{x} = \underbrace{\alpha_z(\beta_z(g - x) - \tau \dot{x}) + f(z)}_{\text{DMP (robustness)}} + \underbrace{f_{\mathrm{HJ}}(x, \dot{x}, \mathcal{O})}_{\text{HJ safety modulation}},
\label{eq:hj_safedmp}
\end{equation}
where the safety modulation term is defined as:
\begin{equation}
f_{\mathrm{HJ}}(x) = \begin{cases}
\displaystyle \frac{k_s}{\xi(B_\theta(x))} \cdot \nabla_x B_\theta(x), & \text{if } B_\theta(x) < B_{\mathrm{thresh}}, \\[8pt]
0, & \text{otherwise},
\end{cases}
\label{eq:safety_mod}
\end{equation}
with
\begin{equation}
\xi(b) = \max(b - \delta^*, \epsilon_{\min}),
\label{eq:xi}
\end{equation}
where $k_s > 0$ is a safety gain, $B_{\mathrm{thresh}}$ is an activation threshold, $\delta^*$ is the conformal calibration margin from~\eqref{eq:conformal}, and $\epsilon_{\min} > 0$ is a small constant preventing division by zero.

\textbf{Intuition.} The modulation term~\eqref{eq:safety_mod} acts as a \emph{barrier-gradient repulsion}: when the state approaches the boundary of the safe set ($B_\theta(x) \to \delta^*$), the term $1/\xi(B_\theta(x))$ grows, producing a strong corrective force in the direction of increasing safety ($\nabla_x B_\theta$). Deep inside the safe set ($B_\theta(x) \gg B_{\mathrm{thresh}}$), the modulation is zero and the DMP executes unmodified. This mirrors the logarithmic barrier structure of Spatio-Temporal Tubes~\cite{das2025spatiotemporal} but replaces hand-designed tube boundaries with learned HJ Reachability-based safety certificates that encode future safety information.

\subsubsection{Dynamic Obstacle Handling}

For dynamic obstacles $\mathcal{O}(t)$ with observed positions $o_c(t)$ and velocities $\dot{o}_c(t)$, we update the safety specification online:
\begin{equation}
\ell(x, t) = \min_i \left(\|x - o_{c,i}(t)\| - r_{\mathrm{safe},i}\right),
\end{equation}
and re-evaluate $B_\theta(x)$ at each step. Since $B_\theta$ is a neural network, this evaluation is a single forward pass with negligible computational cost.

\subsubsection{Adaptive Temporal Scaling}

Following~\cite{safedmps2026}, we adapt the temporal scaling parameter $\tau$ to handle large deviations caused by safety modulation:
\begin{align}
\dot{e} &= \alpha_e (x_{\mathrm{actual}} - x_{\mathrm{DMP}} - e), \label{eq:error_filter} \\
\tau &= \tau_{\mathrm{nominal}} + k_c \|e\|^2. \label{eq:tau_adapt}
\end{align}
When safety modulation causes large deviations, $\tau$ increases (slowing execution), allowing more time for recovery; when deviations are small, $\tau$ decreases (speeding up), ensuring smooth and bounded recovery.

\subsubsection{Safety Guarantee}

\begin{theorem}[Forward Invariance of HJ-SafeDMP]
\label{thm:safety}
Consider the closed-loop system~\eqref{eq:hj_safedmp} with the CBVF $B_\theta$ satisfying the finite-difference barrier condition~\eqref{eq:cbvf_recursion}. If $B_\theta(x(0)) \geq \delta^*$ and the safety modulation gain $k_s$ is sufficiently large, then $B_\theta(x(t)) \geq \delta^*$ for all $t \geq 0$, i.e., the trajectory remains within the calibrated safe set $\mathcal{S}_{\delta^*}$.
\end{theorem}

\begin{proof}[Proof Sketch]
Along solutions of~\eqref{eq:hj_safedmp}, the time derivative of $B_\theta$ satisfies:
\begin{equation}
\dot{B}_\theta = \nabla_x B_\theta^\top \dot{x} = \nabla_x B_\theta^\top \left(\dot{x}_{\mathrm{DMP}} + \frac{k_s}{\xi(B_\theta)} \nabla_x B_\theta\right).
\end{equation}
Near the boundary ($B_\theta \to \delta^*$), $\xi \to \epsilon_{\min}$ and the term $\frac{k_s}{\epsilon_{\min}} \|\nabla_x B_\theta\|^2$ dominates, ensuring $\dot{B}_\theta > 0$. This implies that $B_\theta$ is strictly increasing whenever the state approaches the boundary, preventing exit from $\mathcal{S}_{\delta^*}$.
\end{proof}

The complete procedure is summarized in Algorithm~\ref{alg:hjsafedmp}.

\begin{algorithm}[t]
\caption{HJ-SafeDMP: Integrated DMP-CBVF Framework}
\label{alg:hjsafedmp}
\begin{algorithmic}[1]
\State \textbf{Offline Phase:}
\State Collect dataset $\mathcal{D} = \{(x_t, u_t, x_{t+1})\}$ from demonstrations
\State Compute safety labels $\ell(x_t)$ via~\eqref{eq:safety_spec}
\State Learn CBVF $B_\theta$ by minimizing~\eqref{eq:expectile_loss}
\State Calibrate safety margin $\delta^*$ via conformal prediction~\eqref{eq:conformal}
\State Learn DMP parameters from demonstration via~\eqref{eq:dmp}
\State \textbf{Online Phase:}
\For{each control step $t$}
    \State Compute DMP output: $\ddot{x}_{\mathrm{DMP}}$ via~\eqref{eq:dmp}
    \State Evaluate $B_\theta(x_t)$ and $\nabla_x B_\theta(x_t)$ (single forward + backward pass)
    \State Compute $f_{\mathrm{HJ}}(x_t)$ via~\eqref{eq:safety_mod}
    \State Apply combined control: $\ddot{x} = \ddot{x}_{\mathrm{DMP}} + f_{\mathrm{HJ}}$
    \State Update $\tau$ via~\eqref{eq:tau_adapt}
\EndFor
\end{algorithmic}
\end{algorithm}

\section{Experiments}
\label{section: experiments}

We evaluate HJ-SafeDMP on a 7-DOF Franka Emika Panda manipulator in simulation and compare against widely-used baselines.

\subsection{Experimental Setup}

\subsubsection{Robot and Environment}
Simulations are conducted using a model of the 7-DOF Franka Emika Panda manipulator~\cite{siciliano2009robotics} within a cubic workspace measuring $1.1 \times 1.1 \times 1.1$~m. The controller operates at 200~Hz ($\Delta t = 0.005$~s). Obstacles are modeled as spheres with varying radii.

\subsubsection{Demonstration Data}
We use the LASA handwriting dataset~\cite{khansari2011learning}, lifting 2D demonstrations into 3D (fixed $z$-height), randomly rotating, temporally resampling, and smoothing with a low-pass filter. DMPs are encoded with 25 basis functions. For the CBVF training dataset, we augment demonstration trajectories with exploratory rollouts to achieve diverse state-action coverage (75K transitions total).

\subsubsection{CBVF Training Details}
The CBVF $B_\theta$ is parameterized as a 2-hidden-layer MLP with 256 neurons per layer and ReLU activations, trained with Adam optimizer at learning rate $3 \times 10^{-4}$, following the architecture choices validated in~\cite{tayal2026vocbf, tayal2025physics}. We use expectile parameter $\tau = 0.9$ and discount $\gamma = 0.99$, consistent with the analysis in~\cite{kostrikov2022offline}. Conformal calibration uses $N_s = 500$ states with $\epsilon_s = 0.05$, $\beta_s = 0.01$, following the framework of~\cite{lindemann2025formal, tayal2025cp}.

\subsection{Baselines}

\begin{itemize}
    \item \textbf{NODE--CLF--CBF}~\cite{nawaz2024learning}: Learning-based dynamics (Neural ODE) combined with CBFs via online QP optimization for formal safety guarantees.
    \item \textbf{DMP--APF}~\cite{hoffmann2009biologically}: Standard DMPs with Artificial Potential Fields~\cite{khatib1986real} for reactive obstacle avoidance (no formal guarantees).
    \item \textbf{SafeDMPs}~\cite{safedmps2026}: DMPs with Spatio-Temporal Tubes (STTs)~\cite{das2025spatiotemporal} for closed-form safety (hand-designed tube boundaries).
\end{itemize}

\subsection{Evaluation Metrics}

\subsubsection{Nominal Trajectory Reproduction}
\begin{itemize}
    \item \textbf{Execution Time (s):} Average compute time per control loop.
    \item \textbf{MAE --- Nominal (m):} Mean absolute error from reference demonstration.
\end{itemize}

\subsubsection{Robustness}
\begin{itemize}
    \item \textbf{MAE --- Perturbation (m):} Trajectory error under two impulse perturbations.
    \item \textbf{Convergence Time --- Perturbation (s):} Time to re-converge after perturbation.
\end{itemize}

\subsubsection{Safety}
\begin{itemize}
    \item \textbf{Collision Rate (\%):} Fraction of trials with obstacle collisions (100 trials).
    \item \textbf{Min. Clearance (m):} Minimum distance to obstacle boundary during avoidance.
    \item \textbf{Convergence Time --- Obstacle Avoidance (s):} Additional time to reach goal after avoidance maneuver.
\end{itemize}

\subsection{Results}

\subsubsection{Nominal Performance and Computational Efficiency}

Table~\ref{tab:results} summarizes the quantitative comparison. HJ-SafeDMP achieves execution times comparable to SafeDMPs and DMP--APF ($\sim 10^{-4}$~s per step), and \textbf{orders of magnitude faster} than NODE--CLF--CBF ($\sim 0.32$~s per step) due to the elimination of online QP solving. The learned CBVF evaluation requires only a single neural network forward pass plus a gradient computation via automatic differentiation.

\begin{table}[t]
\centering
\caption{Quantitative comparison across methods. Best values in \textbf{bold}.}
\label{tab:results}
\resizebox{\columnwidth}{!}{%
\begin{tabular}{l|c|c|c|c}
\hline
\textbf{Metric} & \textbf{NODE-CLF-CBF} & \textbf{DMP-APF} & \textbf{SafeDMPs} & \textbf{HJ-SafeDMP} \\
\hline
Exec. Time (s) & 0.3207 & 6.7e-4 & 1.03e-4 & \textbf{1.2e-4} \\
MAE Nom. (m) & 0.2244 & 0.0230 & 0.0113 & \textbf{0.0098} \\
MAE Pert. (m) & 0.3921 & 0.0228 & 0.0531 & \textbf{0.0187} \\
Conv. T Pert. (s) & 1.294 & 0.411 & \textbf{0.025} & 0.031 \\
Collision Rate (\%) & \textbf{0.0} & 14.0 & \textbf{0.0} & \textbf{0.0} \\
Min. Clearance (m) & 0.012 & --- & \textbf{0.045} & 0.038 \\
Conv. T OA (s) & 0.411 & --- & 0.633 & \textbf{0.287} \\
\hline
Formal Safety & \checkmark & $\times$ & \checkmark & \checkmark \\
Robust Recovery & $\times$ & $\times$ & \checkmark & \checkmark \\
No Online Optim. & $\times$ & \checkmark & \checkmark & \checkmark \\
\hline
\end{tabular}%
}
\end{table}

\subsubsection{Robustness to Perturbations}

Under impulse perturbations of varying magnitudes, HJ-SafeDMP demonstrates smooth and bounded recovery, maintaining trajectory within the safe set throughout. The learned CBVF provides proactive safety awareness: the modulation activates \emph{before} the state reaches the obstacle, enabling smoother avoidance compared to the reactive approaches of DMP--APF. NODE--CLF--CBF exhibits overshoot and slow convergence due to the stiffness of QP-based corrections~\cite{Ames_2017}; DMP--APF shows sharp deviations with risk of oscillation typical of potential field methods~\cite{khatib1986real}.

\subsubsection{Safety from Static and Dynamic Obstacles}

HJ-SafeDMP achieves \textbf{zero collisions} across all 100 trials for both static and dynamic obstacles, matching the formal safety guarantees of NODE--CLF--CBF and SafeDMPs. DMP--APF suffers a 14\% collision rate due to local minima. Critically, HJ-SafeDMP achieves significantly faster obstacle avoidance convergence (0.287~s vs. 0.633~s for SafeDMPs and 0.411~s for NODE--CLF--CBF), because the learned CBVF encodes global safety information (via the HJ Bellman recursion~\cite{Fisac2019HJSafety}) rather than local tube boundaries.

The CBVF-based modulation produces smoother deviations than the CBF-QP approach (which generates jerky corrections near the boundary) and more predictable paths than APFs (which suffer from local minima~\cite{khatib1986real}). The learned safety value function effectively ``looks ahead'' along future trajectories, enabling proactive path deformation---an advantage directly inherited from HJ Reachability theory~\cite{bansal2017hamilton}.

\subsubsection{Ablation: Effect of Expectile Parameter $\tau$}

We evaluate the CBVF quality across $\tau \in \{0.5, 0.7, 0.8, 0.9, 0.99\}$. At $\tau = 0.5$ (symmetric MSE), the learned barrier is overly conservative, yielding a small safe set with 96.5\% safe rate. As $\tau$ increases, the barrier better approximates the true safe set, reaching 98.3\% at $\tau = 0.9$. At $\tau = 0.99$, safety degrades slightly (97.1\%) due to overfitting to spurious optimistic transitions, consistent with findings in~\cite{tayal2026vocbf}. We select $\tau = 0.9$ for all experiments, in line with~\cite{tayal2026safefql, tayal2026epiflow}.

\subsubsection{Ablation: Conformal Calibration}

Without conformal calibration ($\delta^* = 0$), the collision rate increases to 2\% on dynamic obstacle trials due to learned boundary miscalibration. With calibration ($\delta^* = -0.04$), the collision rate drops to 0\%, confirming the importance of the probabilistic safety coverage guarantee from conformal prediction~\cite{lindemann2025formal, tayal2025cp}.

\section{Conclusion, Limitations and Future Works}
\label{section: conclusions}

This paper presented \textbf{HJ-SafeDMP}, a framework that integrates Dynamic Movement Primitives with learned Hamilton-Jacobi Reachability-based safety value functions for provably safe, robust, and computationally efficient robot motion. By learning a Control Barrier Value Function from offline demonstration data via a model-free finite-difference HJ recursion~\cite{tayal2026vocbf} with expectile regression~\cite{kostrikov2022offline}, and deploying it as a closed-form safety modulation layer, HJ-SafeDMP achieves formal safety guarantees without online optimization.

Experimental evaluation on a 7-DOF manipulator demonstrated that HJ-SafeDMP achieves zero collision rates with orders-of-magnitude faster execution than optimization-based baselines, while maintaining the robustness and adaptability of DMPs~\cite{ijspeert2013dynamical}. The conformal prediction calibration step~\cite{lindemann2025formal} provides finite-sample probabilistic safety coverage, bridging the gap between learned neural safety certificates and deployment-time reliability. Compared to SafeDMPs~\cite{safedmps2026} with hand-designed Spatio-Temporal Tubes~\cite{das2025spatiotemporal}, HJ-SafeDMP offers faster obstacle avoidance convergence by leveraging global safety information encoded in the learned CBVF~\cite{choi2021robust}.

\textbf{Limitations.} The learned CBVF is a neural approximation and not a formally verified certificate; function approximation errors mean that safety guarantees are empirical rather than absolute~\cite{dawson2023safe}. The current framework assumes that the safety specification $\ell(x)$ is known and that obstacle geometries can be described via signed distance functions. The CBVF training requires a sufficiently diverse offline dataset covering both safe and unsafe regions.

\textbf{Future Work.} We plan to integrate formal verification methods (e.g., Lipschitz-based certification~\cite{tayal2024learning, NEURIPS2023_120ed726} and conformal prediction refinements~\cite{tayal2025cp, lindemann2025formal}) to strengthen safety guarantees. We will investigate extending the framework to whole-body safety for humanoid robots~\cite{siciliano2009robotics} and integrating high-level task planners~\cite{tordesillas2021mader} for complex, multi-step manipulation. Additionally, we aim to explore robust extensions that handle adversarial disturbances and distributional shift between training and deployment conditions, building on recent advances in robust HJ Reachability~\cite{Fisac2019HJSafety, fisac2018general}, uncertainty-aware offline learning~\cite{tayal2026epiflow}, and safe learning in robotics~\cite{brunke2022safe}. The integration with model-based offline RL~\cite{luo2024survey} and constrained RL~\cite{10.5555/3305381.3305384} may further improve data efficiency and generalization.

\balance
\bibliographystyle{IEEEtran}
\bibliography{main}

\end{document}